\documentclass{article}
\usepackage[margin=1.1in]{geometry}
\usepackage{natbib}
\usepackage[algo2e, vlined, ruled, linesnumbered]{algorithm2e}
\usepackage{float}

\newcommand{\OO}{\mathcal{O}}
\newcommand{\tOO}{\wt{\OO}}

\def\poly{\mathop{\mbox{poly}}}

\usepackage{pifont}

\newcommand{\siprod}[2]{\langle#1,#2\rangle}

\newcommand{\wh}{\widehat}
\newcommand{\wt}{\widetilde}

 \newcommand{\A}{\mathcal{A}}

\newcommand{\unif}{\mathsf{Unif}}

\usepackage{makecell}

\newcommand{\SEP}{\mathsf{SEP}}
\newcommand{\hattheta}{\widehat{\theta}}
\newcommand{\inner}[1]{\left\langle #1 \right\rangle}

\newcommand{\hatq}{\widehat{q}}
\newcommand{\ftrl}{\textbf{\textup{FTRL}}\xspace}
\newcommand{\bias}{\textbf{\textup{Bias}}\xspace}
\newcommand{\bonus}{\textbf{\textup{Bonus}}\xspace}
\newcommand{\calA}{\mathcal{A}}
\newcommand{\dd}{\mathrm{d}}
\newcommand{\hatSigma}{\widehat{\Sigma}}
\newcommand{\hatx}{\widehat{x}}
\newcommand{\calD}{\mathcal{D}}
\newcommand{\hatPi}{\widehat{\Pi}}
\newcommand{\empD}{\widetilde{\calD}}
\newcommand{\alg}{\mathsf{ALG}\xspace}
\newcommand{\disc}{\Pi_{\epsilon}}

\newcommand{\lin}{\mathrm{lin}}
\newcommand{\gr}{\text{gr}}
\newcommand{\sm}{\text{sm}}

\newcommand{\calS}{\mathcal{S}}

\newcommand{\cw}[1]{{\color{red}CW:#1}}

\newcounter{protocol}

\usepackage{amsmath}
\usepackage{amsthm}
\usepackage{amssymb}
\usepackage{amssymb}
\usepackage{thmtools,thm-restate}
\usepackage{graphicx}
\usepackage{multirow}

\declaretheorem{lemma}

\DeclareMathOperator*{\argmax}{arg\,max}
\DeclareMathOperator*{\argmin}{arg\,min}
\DeclareMathOperator*{\Exp}{\mathbb{E}}  
\DeclareMathOperator*{\Rad}{Rad}         
\DeclareMathOperator*{\Ndim}{NatDim}       

\newtheorem{theorem}{Theorem}
\newtheorem{assumption}{Assumption}

\newtheorem{corollary}{Corollary}
\newtheorem{definition}{Definition}

\newcommand{\vertex}{\mathrm{vert}}
\newcommand{\norm}[1]{\left\lVert #1\right\lVert}
\newcommand{\ind}[1]{\mathbb{I}\left\{#1\right\}}
\newcommand{\Reg}[1]{\mathrm{Reg}_{T}(#1)}
\newcommand{\X}{\mathcal{X}}

\newcommand{\E}[1]{\Exp\left[#1\right]}

\newcommand{\iprod}[2]{\left\langle #1,#2\right\rangle}
\newcommand{\reals}{\mathbb{R}}

\newcommand{\Cclass}{\mathcal{C}}     
\newcommand{\Pilin}{\Pi_\textrm{lin}}

\newcommand{\CW}[1]{{\color{orange}[CW:#1]}}

\newcommand{\julia}[1]{{\color{purple}[JO:#1]}}
\newcommand{\jack}[1]{{\color{teal}[JM:#1]}}

\usepackage[utf8]{inputenc} 
\usepackage[T1]{fontenc}    
\usepackage{hyperref}       
\usepackage{url}            
\usepackage{booktabs}       
\usepackage{amsfonts}       
\usepackage{nicefrac}       
\usepackage{microtype}      
\usepackage{xcolor}         
\usepackage{amsthm}

\usepackage{times}

\title{An Improved Algorithm for\\ Adversarial Linear Contextual Bandits via Reduction}

\author{%
  Tim van Erven\thanks{University of Amsterdam. Email: \texttt{tim@timvanerven.nl}} \and
  Jack Mayo\thanks{University of Amsterdam; Kurtos.ai. Email: \texttt{jackjamesmayo@gmail.com}} \and
  Julia Olkhovskaya\thanks{Delft University of Technology. Email: \texttt{julia.olkhovskaya@gmail.com}} \and
  Chen-Yu Wei\thanks{University of Virginia. Email: \texttt{chenyu.wei@virginia.edu}}
}

\date{June 1, 2026$^{**}$ }

\DeclareMathOperator{\conv}{conv}
\DeclareMathOperator{\cone}{cone}

\newcommand{\approxlcb}{\widehat{\Omega}}    
\newcommand{\todo}[1]{{\color{magenta} TODO: #1}}

\newcommand{\tim}[1]{{\color{magenta} Tim: #1}}

\renewcommand{\tim}[1]{}
\renewcommand{\jack}[1]{}
\renewcommand{\cw}[1]{}

\begin{document}

\maketitle
{\renewcommand{\thefootnote}{$**$}\footnotetext{An earlier version of this work appeared at NeurIPS 2025. Compared to the earlier version, the current version removes the assumption of $\calA_t$ being a polytope with a finite number of constraints, and only assumes a linear optimization oracle for $\calA_t$. }}

\begin{abstract}%
We present an oracle-efficient, near-optimal algorithm for linear contextual bandits with adversarial losses and stochastic action sets, only requiring a linear optimization oracle for the action sets in each round. Our approach reduces this setting to misspecification-robust adversarial linear bandits with fixed action sets. Without knowledge of the context distribution or access to
a context simulator, the algorithm achieves $\tOO(\min\{d^2\sqrt{T}, \sqrt{d^3T\log K}\})$ regret and runs in $\poly(d,T)$ time plus $\poly(d,T)$ calls to the linear optimization oracles, where $d$ is the feature dimension,  $K$ is an upper bound on the number of actions in each round, and $T$ is number of rounds.  This resolves the open question by \citet{liu2024bypassing} on whether one can obtain $\poly(d)\sqrt{T}$ regret in polynomial time independent of the number of actions.  For the important class of combinatorial bandits with adversarial losses and stochastic action sets, our algorithm is the first to achieve $\poly(d)\sqrt{T}$ regret in polynomial time, while no prior algorithm achieves even $o(T)$ regret in polynomial time to our knowledge. When a simulator is available, the regret bound can be improved to $\tOO(d\sqrt{L^\star})$, where $L^\star$ is the cumulative loss of the best policy.
\end{abstract}


\section{Introduction}

We consider the following linear contextual bandit problem: At each round $t = 1, \ldots, T$, the environment generates a hidden loss vector $\theta_t \in \mathbb{R}^d$ and an action set $\mathcal{A}_t \subset \mathbb{R}^d$. The learner observes $\mathcal{A}_t$, selects an action $a_t \in \mathcal{A}_t$, and incurs loss $a_t^\top \theta_t$. The goal is to compete with the best fixed policy—defined as a mapping from an action set to an element in it. This setting generalizes the classical linear bandit model by allowing the action sets $\mathcal{A}_t$ to vary stochastically over time. Crucially, each $\mathcal{A}_t$ encodes the \emph{context} based on which the learner makes decisions. In this work, $\mathcal{A}_t$ is called \emph{context} or \emph{action set} interchangeably. 

This framework is applicable in settings such as healthcare and
recommendation systems, where decisions must be made conditional on
context. Prior work on contextual bandits has studied a variety of
assumptions on how losses and contexts are generated. While much of the
literature assumes i.i.d. losses and arbitrarily chosen action sets (for
which a well-known algorithm is LinUCB \citep{li2010contextual}), we focus on the complementary regime: the action sets are drawn i.i.d. from a fixed distribution $\calD$, while the losses may be chosen adversarially.

A first computationally efficient algorithm for this setting was
proposed by \citet{neu20} under the assumption that the context (i.e.,
action set) distribution is \emph{known}. Since an action set is a
subset of $\mathbb{R}^d$ (i.e., it lies in the space
$2^{\mathbb{R}^d}$), the distribution over action sets is in the space
$\Delta(2^{\mathbb{R}^d})$, which is generally intractable to represent
efficiently. This assumption was later removed by subsequent efficient
algorithms \citep{luo2021policy, sherman2023improved, dai2023refined,
liu2024bypassing}. In the setting where the learner has access to a
simulator that can generate free contexts (i.e., the learner is able to
sample contexts from $\calD$ as many times as they want without
incurring cost), \citet{dai2023refined} shows that a near-optimal regret
bound of $\widetilde{\mathcal{O}}(\min \{d\sqrt{T}, \sqrt{dT\log K}\})$
is achievable, where $K=\max_t |\calA_t|$.
\tim{How can we reconcile the
claim that $O(d\sqrt{T})$ is the optimal rate with the rate
$\sqrt{dT\log K}$ in the second row of the table?}\jack{I think the claim here is simply that it should be the near-optimal action set-independent regret bound. We know that in the linear CB case (similar to the linear bandit case), one can get a K-free or K dependent rate, trading off for a factor of d}\cw{changed the lower bound to the minimum of the two}
When the learner
has neither
knowledge of the context distribution nor simulator access to random
context samples, 
the best known
results are by \citet{liu2024bypassing}: they provide an algorithm with
near-optimal regret $\tOO(d\sqrt{T})$ with run time $T^{\Omega(d)}$ and
another algorithm with regret $\tOO(d^2 \sqrt{T})$ with run time
$\poly(d, K, T)$.  Notably, while the regret
bound of this last algorithm is independent of the number of actions
$K$,  its computational complexity scales polynomially in $K$. In fact,
this is the case for all previous algorithms as well
\citep{luo2021policy, sherman2023improved, dai2023refined,
liu2024bypassing, liu2024towards}. This makes them unsuitable for many important combinatorial problems (e.g., $m$-set, shortest path, flow, bipartite matching), where $K$ is usually exponentially large in the dimension $d$. 
\renewcommand{\arraystretch}{1.5}
\begin{table}[t]\label{tab:comparison}
   \center
   \caption{Comparison with state-of-the-art results in adversarial
   linear contextual bandits. $d$ is the feature dimension, and $K$ is an
 upper bound on the number of actions. }
   \small
   \begin{tabular}{|c|c|c|c|c|}
       \hline 
       Algorithm & Regret & Computation & Simulator & Feedback  \\
       & (omitting $\log(dT)$ factors) &  & & \\
       \hline 
       \cite{schneider2023optimal}& \makecell{$\sqrt{dT}$ for the special case \\
       $\calA_t\subseteq \{\mathbf{e}_1, \ldots, \mathbf{e}_d\}$
       } & $\poly(d,T)$ & no & bandit \\ 
       \hline
       \cite{neu2014online} & $(dT)^{2/3}$ & \makecell{$\poly(d,T)$ plus \\ $\poly(d,T)$ oracle calls} & no & semi-bandit \\ 
       \hline
       \cite{dai2023refined} & $\min\{d\sqrt{T}, \sqrt{dT\log K}\}$ & $\poly(d, K, T)$ &  yes & bandit \\
       \hline
       \cite{liu2024bypassing} & $d\sqrt{T}$ & $K\cdot T^{\Omega(d)}$ &  no & bandit \\
       \hline
       \cite{liu2024bypassing}  & $d^2\sqrt{T}$ & $\poly(d,K,T)$ & no & bandit \\
       
       \hline
         Ours  & $\min\{d^2\sqrt{T}, \sqrt{d^3T\log K}\}$ 
       & \multirow{2}{*}{\makecell{$\poly(d,T)$ plus \\ $\poly(d,T)$ oracle calls}} 
       & no & bandit \\ 
       \cline{1-2}\cline{4-5}
       Ours  & $d\sqrt{L^\star}$ 
       &  
       & yes & bandit \\ 
       \hline
   \end{tabular}
\end{table}

Our work gives the first algorithm whose computational complexity does not explicitly
scale with the number of actions, making adversarial linear contextual
bandits applicable to a much wider range of problems. Assuming access to linear optimization oracles for the individual action sets in each round, without simulator
access, our method achieves regret $\tOO(\min\{d^2 \sqrt{T}, \sqrt{d^3T\log K}\})$ and with
simulator access this can be improved to
$\tOO(d\sqrt{L^\star})$, where $L^\star = O(T)$ is the cumulative loss of the best policy.
Our algorithm runs in $O(d, T)$ time plus $O(d,T)$ calls to the linear optimization oracle. 

In many combinatorial problems, a linear optimization oracle runs in
$O(d)$ time, while the number of actions is $K=2^{\Omega(d)}$. For
example, this is the case for the shortest-path problem, where $d$ is the number of edges and $K$ is the number of paths.  
For combinatorial problems with stochastic action sets and adversarial
losses, we are only aware of \cite{neu2014online} who studied the case
where the learner has \emph{semi-bandit} feedback. Their algorithm
achieves $\tOO\big((dT)^{2/3}\big)$\tim{$O$ or $\tOO$?} regret with $\poly(d,T)$ calls to the linear optimization oracle. Compared to their work, our work
weakens the assumption on the feedback (from semi-bandits to
full-bandits) and improves the regret bound (from $T^{2/3}$ to $\sqrt{T}$).

It remains open how to efficiently achieve the near-optimal bound $\tOO(\min\{d\sqrt{T}, \sqrt{dT\log K}\})$ without simulators, even when the learner is willing to pay $\poly(d,K,T)$ computation.  This has only been resolved in the special case where action sets are subsets of the standard basis $\{\mathbf{e}_1, \ldots, \mathbf{e}_d\}$, which is also known as the sleeping bandits problem  \citep{schneider2023optimal}.   



Our result is achieved by establishing a novel and computationally
efficient reduction from \emph{adversarial linear contextual bandits} to
\emph{adversarial linear bandits with misspecification}. Linear bandits
can be viewed as linear contextual bandits with a fixed context and is
therefore less challenging than linear contextual bandits. This
reduction scheme is related to the work of \cite{hanna2023contexts}, which reduces stochastic linear contextual bandits (with
both stochastic contexts and stochastic losses) to misspecified
stochastic linear bandits.  Compared to \cite{hanna2023contexts}, we make two key advances: First, our reduction is polynomial-time and operates under the mild assumption of having access to a linear optimization oracle for individual action sets, whereas \cite{hanna2023contexts} focus solely on the regret bound and it remains open how their reduction can be implemented in polynomial time. Second, we study the strictly more general setting with adversarial losses. A more detailed comparison with \cite{hanna2023contexts} is provided in Appendix~\ref{app: comparison}.  

Finally, regarding results in terms of $L^\star$, we remark that our bound
of $\tOO(d\sqrt{L^\star})$ can be specialized to the setting of
\cite{olkhovskaya2024first} (a slightly different formulation of
adversarial contextual bandits). We then obtain a slightly worse rate but, importantly, remove their unrealistic assumption that the context distribution is log-concave.

\tim{Make LaTeX place this table earlier}
\tim{A bit weird to omit log factors except for one algorithm. I suspect
this is because $K$ may be exponential in $d$, so then $\log K$ factors
matter. But then we have to report $\log K$ factors for all algorithms!}\cw{changed}
\section{Problem Description and Main Results} 

\paragraph{Notation}
Suppose $\calA\subset \reals^d$ is a set of vectors. Then
the convex hull of $\calA$ is denoted as $\conv(\calA)=\{
x=\sum_{i=1}^{k}\lambda_{i}a_{i} : k\in \mathbb{N}, \sum_{i=1}^k \lambda_i=1, \lambda_i\geq 0, a_{i}\in \calA
\}$. If $\calA$ is a polytope, we denote its set of vertices as
$\vertex(\calA)$, and denote the normal cone of $\calA$ at a vertex $v\in\vertex(\calA)$ as $\mathcal{N}(\calA,v)=\{y\in \reals^d:\max_{x\in \calA}\iprod{y}{x-v}\leq 0\}$. 

\subsection{Linear Contextual Bandits}
\label{sec:setting}


    


For simplicity, we consider an oblivious adversary.
Before any interaction with the learner, the adversary secretly chooses $T$ loss vectors $(\theta_{t})_{t\in[T]}$ where $\theta_{t}\in \mathbb{R}^d$ for all $t$. 
For each round $t=1,2,\ldots, T$, an action set $\calA_t\subset
\mathbb{R}^d$ is drawn according to
$\calA_{t}\overset{\text{i.i.d.}}{\sim} \mathcal{D}$ and revealed to the
learner. The learner then  chooses an action $a_t\in\calA_t$ and
observes loss $\ell_t\in[0,1]$ with $\mathbb{E}_t[\ell_t]= a_t^\top
\theta_t$, where $\mathbb{E}_t$ denotes the expectation conditioned on the history up to time $t-1$ and~$a_t$.
The expected regret against a fixed policy $\pi\in \Pi$ is defined as 
\begin{equation*}
\Reg{\pi}=\E{\sum_{t=1}^{T}\inner{a_t - \pi(\calA_t), \theta_t}},
\end{equation*}
where policy $\pi$ maps an action set $\calA\subset \mathbb{R}^d$ to a point $\pi(\calA)\in\conv(\calA)$. \footnote{Defining $\pi(\calA)\in \conv(\calA)$ instead of $\pi(\calA)\in\calA$ only makes the guarantee more general and simplifies the notation. Equivalently, it defines a randomized policy: to execute $\pi(\calA)$, sample $a\in\calA$ such that $\mathbb{E}[a] = \pi(\calA)$. } We assume $\calA \subset \mathbb{B}_2(1)$ for $\calA$ in the support of $\calD$, where $\mathbb{B}_2(r)$ is the Euclidean ball with radius $r$. Let $\Theta^\star = \{\theta\in\mathbb{R}^d:~ \sup_{\calA\in \mathrm{supp}(\calD)} \sup_{a\in\calA}|\inner{a,\theta}|\leq 1 \}$ be the space of possible $\theta_t$. By Lemma~8 of \cite{wei2021learning}, we assume $\Theta^\star\subset\mathbb{B}_2(\sqrt{d})$ without loss of generality. 
%
%
%


Note that the policy that minimizes the total expected loss is 
\begin{align*}
   \argmin_{\pi\in\Pi} \mathbb{E}\left[\sum_{t=1}^T \inner{\pi(\calA_t), \theta_t}\right]= \argmin_{\pi\in\Pi} \mathbb{E}_{\calA\sim \calD}\left[\inner{\pi(\calA), \sum_{t=1}^T \theta_t}\right],
\end{align*} 
which is attained by the policy $\pi(\calA) = \argmin_{a\in\calA} \big\langle a,\sum_{t=1}^T \theta_t\big\rangle$. Thus, to minimize the expected regret, it suffices to compare the learner to the class of linear classifier policies 
\begin{equation}\label{eq:linclass}
    \Pilin=\left\{\pi_{\phi}~\big|~ \phi \in \reals^{d}\right\} \text{ where }\pi_{\phi}(\calA)\in\argmin_{a\in\calA}\iprod{a}{\phi}.   
\end{equation}

In this work, we assume the learner can compute $\pi_\phi(\calA)$
efficiently for any $\phi$ and any $\calA$ in the support of $\calD$
through the following oracle:
\begin{assumption}[Linear optimization oracle]\label{assum: linear optimization}
    The learner has access to a linear optimization oracle that takes an
    action set $\calA$ in the support of $\calD$ and a non-zero vector
    $\phi\in \mathbb{R}^d$ as inputs, and returns $\pi_\phi(\calA)\in
    \argmin_{z\in \calA} \inner{z,\phi}$. Any ties are broken in an
    arbitrary, but fixed way; i.e.,\ the same point is returned every time the
    oracle is called with the same action set $\calA$.  
\end{assumption}


%
%

\subsection{Linear Bandits and $\epsilon$-Misspecified Linear Bandits}
The adversarial linear bandit problem with oblivious adversary 
is the case in which the adversary decides $(\theta_t)_{t\in[T]}$ before any interaction with the learner. The learner is given a fixed action set $\approxlcb\subset \mathbb{R}^d$. At every round $t\in[T]$, the learner chooses an action $y_{t}\in\approxlcb$ and receives 
$c_t\in[0,1]$ as feedback with $\mathbb{E}_t[c_t]=\iprod{y_{t}}{\theta_{t}}$. The regret with respective to a fixed action $y\in\approxlcb$ is defined as 
\begin{align*}
   \Reg{y} = \E{\sum_{t=1}^T \inner{y_t - y, \theta_t}}. 
\end{align*}
A \emph{misspecified} linear bandit problem is the case where the learner, instead of receiving an unbiased sample of $\inner{y_t,\theta_t}$ as feedback, receives $c_t$ with $|\mathbb{E}_t[c_t]-\inner{y_t, \theta_t}|\leq \epsilon$ for some $\epsilon$ known to the algorithm. 


\subsection{Results Overview}\label{sec: poly time}

In this section, we present a general framework that can reduce the adversarial contextual bandit problem to a misspecification-robust linear bandit algorithm defined as the following.  

\begin{definition}[$\alpha$-misspecification-robust adversarial linear bandit algorithm]\label{def: misspec}
A $\alpha$-misspecification-robust linear bandit algorithm over action set $\approxlcb\subset \mathbb{R}^d$ has the following property: with a given 
$\epsilon>0$ and the guarantee that every time the learner chooses $y_t\in \approxlcb$, the loss received 
\tim{Don't use $\ell_t$ for the feedback here, because it is already used for
the loss in the original contextual bandit setting.}\cw{changed to $c_t$}
$c_t\in[0,1]$ satisfies 
\tim{Should this be $\mathbb{E}_t$ (everywhere in the discussion of the bias)?}\cw{Yes, changed. }
\begin{align*}   
   |\mathbb{E}_t[c_t] - \inner{y_t, \theta_t}| \leq \epsilon, 
\end{align*}
the algorithm ensures 
\begin{align}
   \mathbb{E}\left[\sum_{t=1}^T  \inner{y_t, \theta_t}\right] \leq \min_{y\in \approxlcb}\sum_{t=1}^T \inner{y, \theta_t} + \widetilde{O}\left( d\sqrt{T} + \alpha\sqrt{d}\epsilon T \right).  \label{eq: LB bound}
\end{align}
\end{definition}
Notice that there is an $\alpha$ parameter in Definition~\ref{def: misspec} that specifies the dependence of the regret on the misspecification level $\epsilon$. It is known that $\alpha=1$ is the statistically optimal dependence. However, for specific algorithms, we might have $\alpha>1$. 

We establish the following reduction:



\begin{theorem}\label{thm: thm without }
    Given access to an $\alpha$-misspecification-robust adversarial
    linear bandit algorithm, we can achieve
    $\min_{\pi\in\Pi}\Reg{\pi}\leq \widetilde{O}(d\sqrt{T}+\alpha
    d\sqrt{T\log K})$ in adversarial linear contextual bandits without
    access to a simulator. 
\end{theorem}
We remark that the $\alpha d\sqrt{T\log K}$ \tim{This rate changed} term in Theorem~\ref{thm: thm without } comes from the misspecification term $\alpha\sqrt{d}\epsilon T$ in \eqref{eq: LB bound}. When the learner has access to a simulator that generates free contexts, the $\epsilon$ can be made arbitrarily small, allowing us to achieve the optimal $d$ dependence. This will be discussed in Section~\ref{sec: small loss bound}. 

\section{Reduction from Linear Contextual Bandits to Linear Bandits}\label{sec: linear bandit}
Let $\pi$ denote a policy, which maps any given action set $\calA$ to a
randomized action in $\conv(\calA)$. 
Let $\Pi$ denote
the set of all possible policies. We define the following map
  \begin{equation*}
  \Psi(\pi) = \mathbb{E}_{\calA\sim \mathcal{D}}{\left[\pi(\calA)\right]}, 
   \end{equation*}
   which is the mean action of $\pi$.
  Applying $\Psi$ to all $\pi \in \Pi$, the set $\Omega = \{\Psi(\pi)
  \mid \pi \in \Pi\}$ is induced. Note that, under this map, 
 the expected loss under actions drawn such that $\mathbb{E}[a_t]=\pi(\calA_t)$ may be written as
\begin{align*}
    \mathbb{E}\left[\inner{a_t, \theta_t}\right] = \mathbb{E}_{\calA_t\sim \calD} \left[\inner{\pi(\calA_t), \theta_t}\right] = \inner{\Psi(\pi), \theta_t}. 
\end{align*}

Accordingly, 
if the learner draws $a_t$ from policy $\pi_t$ in round $t$, the expected
regret may be written as
 \begin{align*}
 	\Reg{\pi}&=\E{\sum_{t=1}^{T}\iprod{a_t - \pi(\calA_t)}{\theta_{t}}} = \E{\sum_{t=1}^{T}\iprod{\pi_t(\calA_t) - \pi(\calA_t)}{\theta_{t}}} 
 =\E{\sum_{t=1}^{T}\iprod{\Psi(\pi_{t}) - \Psi(\pi)}{\theta_{t}}}.  	
 \end{align*}
\subsection{Approximating $\Omega$}

Since $\Omega$ cannot be accessed directly without full knowledge of the
context distribution $\mathcal{D}$, we cannot work with $\Omega$ directly.
Instead, we will therefore approximate $\Omega$ by its empirical
counterpart $\approxlcb$ based on $N$ i.i.d.\ samples $\calA_{1},\dots,
\calA_{N}$ from $\mathcal{D}$:
\begin{align}
\approxlcb = \left\{\hat\Psi(\pi)\Big{|} \pi \in \Pi\right\} &=
\left\{x\in \reals^{d}:x=\frac{1}{N}\sum_{i=1}^{N}a_{i} ~\Bigg|~
a_{i}\in \conv(\calA_{i}), a_i = a_j \text{ when } \A_i
= \A_j\right\},   \label{eqn:approxlcb}
\end{align}
where $\hat\Psi(\pi) = \frac{1}{N} \sum_{i=1}^N \pi(\calA_i)$. It turns
out the constraint that $a_i = a_j$ whenever $\A_i
= \A_j$ actually does not do anything:
\begin{restatable}{lemma}{approxlcbminksum}\label{lem:approxlcb-minksum}
$\approxlcb = \frac{1}{N}\sum_{i=1}^N \conv(\calA_i)$ (where the sum is the Minkowski sum).
\end{restatable}
The proof is in Appendix~\ref{app:approxlcb-minksum}.

We proceed to show that the empirical cumulative loss of any linear
classifier policy on the sample $\calA_1,\ldots,\calA_N$ is close to its
expected cumulative loss, as long as $N$ is large enough:
\begin{restatable}[Uniform Convergence]{lemma}{uniformconvergence}\label{lem:uniform_convergence}
  Consider any loss vector $\theta \in \reals^d$, and suppose that
  $|\calA| \leq K$ and $\max_{a \in \calA} |\langle a, \theta \rangle|
  \leq b$ almost surely. Then, for any $\delta \in (0,1]$, uniformly
  over all linear classifier policies $\pi_\phi$, the difference in
  performance of $\pi_\phi$ on the sample $\calA_1,\ldots,\calA_N$ and
  its expected performance  is at most
  \[
    \sup_{\phi \in \reals^d}
    \Big|
    \big\langle \Psi(\pi_\phi), \theta \big\rangle
    - \big\langle \hat \Psi(\pi_\phi),\theta\big\rangle
    \Big|
      \leq
      2 b\sqrt{\frac{2 d\ln (NK^2)}{N}}
      + b\sqrt{\frac{2 \ln(4/\delta)}{N}}
  \]
  with probability at least $1-\delta$.
\end{restatable}
The proof is provided in Appendix~\ref{app:proof_uniform_convergence}.
Its key idea is to rephrase the result as an equivalent statement about
uniform convergence for linear multiclass classifiers with $K$ classes
in the batch setting, with an unusual loss function. We can then obtain
a concentration inequality that holds uniformly over all linear
classifiers using standard tools. Specifically, we go via Rademacher
complexity and a bound on the growth function of the class of multiclass
linear classifiers in terms of its Natarajan dimension, which is known
to be at most~$d$.

  %

\subsection{Connection between Linear Contextual Bandits and Linear Bandits}\label{sec: connection con}
\begin{algorithm2e}
    \caption{Adversarial Linear Contextual Bandits} \label{alg: lcb alg}
    \textbf{Input}: An adversarial linear bandit algorithm $\alg$ and a set $\calS$ of $N$ action sets drawn from $\calD$. \\
    Initiate an instance of $\alg$ over the action set $\approxlcb$ constructed from $\calS$ (according to Lemma~\ref{lem:approxlcb-minksum}).  \\
    \For{$t=1, \ldots, T$}{
         Obtain $y_t$ from $\alg$.  \\
         Let $\Phi_t\subset \mathbb{R}^d$ be the output of PolicyDecomposition (Algorithm~\ref{alg: decompose}) with inputs $(\calS, y_t, N)$. \label{line: 55}\\
         Sample $\phi_t\sim \unif(\Phi_t)$. \label{line: 66}\\
         Receive action set $\calA_t$, choose action $a_t=\pi_{\phi_t}(\calA_t)\in\argmin_{a\in\calA_t} \inner{a, \phi_t}$, and receive loss $\ell_t\in[0,1]$. \label{line: 77} \\
         Send $\ell_t$ to $\alg$. 
    }
\end{algorithm2e}

\begin{algorithm2e}[t]
    \caption{PolicyDecomposition $(\calS, y, M)$} \label{alg: decompose} 
    \textbf{Input}: 
    $
        \calS =\{\calA_1, \ldots, \calA_N\}  \text{ containing $N$ action sets},\  
        y \in \approxlcb \triangleq \frac{1}{N}\sum_{i=1}^N \conv(\calA_i),\  M \in\mathbb{N}. 
    $ \\
    Initialize $z_0=0$. \\ 
    \For{$j=1,2,\ldots, M$}{
         $\phi_j = z_{j-1} - y$. \\[2pt]
         $v_j = \frac{1}{N}\sum_{i=1}^N \pi_{\phi_j}(\calA_i)$ \hfill \text{// calling the linear optimization oracle} \\[2pt]
         $z_j = (1-\frac{1}{j})z_{j-1} + \frac{1}{j}v_j$. 
    }
    \textbf{return} $\Phi=\{\phi_1, \ldots, \phi_M\}$.  
\end{algorithm2e}
The procedure that reduces linear contextual bandits to linear bandits
is outlined in Algorithm~\ref{alg: lcb alg}. 
Ignoring the discrepancy between $\approxlcb$ and $\Omega$ for now, according to
Section~\ref{sec: linear bandit}, we can view the problem as linear
bandits over $\approxlcb \subset \mathbb{R}^d$. When the linear bandit algorithm suggests a point $y_t\in\approxlcb$ to play, the learner should feed it with a loss with expectation $\inner{y_t, \theta_t}$.  A key question is: when the linear bandit
algorithm gives a point $y_t\in\approxlcb$, what corresponding
policy $\pi_t$ should the learner use to interact with the world?  
%

Recall that every point $y_t\in\approxlcb$ corresponds to some policy in the global policy space $\Pi$, so a naive solution is to pick any policy $\pi_t\in\Pi$ such that $\hat{\Phi}(\pi_t)=y_t$ in round $t$.  However, as $\hat{\Phi}(\pi_t)$ and $\Phi(\pi_t)$ may be very different from each other, the loss feedback (whose expectation is $\inner{\Phi(\pi_t), \theta_t}$) may not be close to the required feedback $\langle \hat{\Phi}(\pi_t), \theta_t\rangle$ for the linear bandit algorithm. Fortunately, if $\pi_t$ is a \emph{linear classifier policy}, by Lemma~\ref{lem:uniform_convergence}, the difference between $\inner{\Phi(\pi_t), \theta_t}$ and $\langle \hat{\Phi}(\pi_t), \theta_t\rangle$ will be at a tolerable level. Therefore, we aim to make the learner only use linear classifier policies. 


Linear classifier policies of the form $\pi_\phi$ always choose an
action in $\argmin_{a\in\calA} \inner{a, \phi}$ given action set
$\calA$. This policy corresponds to the point
$\hat{\Psi}(\pi_\phi)\in\frac{1}{N}\sum_{i=1}^N
\argmin_{a_i\in\calA_i}\inner{a_i, \phi} \in \argmin_{v\in \approxlcb}
\inner{v, \phi}$ (according to Lemma~\ref{lem:approxlcb-minksum}), which is always
on the boundary of $\approxlcb$. Therefore, in order to execute
$y_t\in\approxlcb$ as linear classifier policies, we need to
decompose it into a convex combination of points on the boundary of
$\approxlcb$ that all correspond to linear classifier policies. To achieve this, we use a Frank-Wolfe algorithm, as outlined in Algorithm~\ref{alg: decompose}. Its guarantee is given by the following lemma. 

\begin{restatable}{lemma}{frankwolfe}\label{lem: Frankwolfe}
    Let $\hat{\Psi}(\pi)$ be as defined in \eqref{eqn:approxlcb}.  
    Then Algorithm~\ref{alg: decompose} with inputs $(\{\calA_1, \ldots, \calA_N\}, y, M)$ ensures
    \begin{align*}
        \left\|y - \frac{1}{M}\sum_{j=1}^M \hat{\Psi}(\pi_{\phi_j})\right\| \leq \frac{2}{\sqrt{M}},
    \end{align*}
    where $\{\phi_1, \ldots, \phi_M\}$ is the output of Algorithm~\ref{alg: decompose}.
\end{restatable}
See Appendix~\ref{app: Frankwolfe} for the proof. Lemma~\ref{lem:
Frankwolfe} provides the following concrete way to execute the point $y_t\in \approxlcb$: First, obtain $\{\phi_1, \ldots, \phi_M\}$ by calling Algorithm~\ref{alg: decompose} with $y_t$. Then, execute the \emph{mixture} of linear policies by first sampling $j\sim \unif(\{1,2,\ldots,M\})$ and then executing the linear classifier policy $\pi_{\phi_j}$. The expected loss under this procedure~is 
\begin{align*}
   \frac{1}{M}\sum_{j=1}^M \inner{\Psi(\pi_{\phi_j}),  \theta_t} \approx \frac{1}{M}\sum_{j=1}^M \inner{\hat{\Psi}(\pi_{\phi_j}), \theta_t}  \approx \inner{y_t,  \theta_t}, 
\end{align*}
where the first ``$\approx$'' is by Lemma~\ref{lem:uniform_convergence} and the second is by Lemma~\ref{lem: Frankwolfe}. This is performed in Line~\ref{line: 55}--\ref{line: 77} in Algorithm~\ref{alg: lcb alg}. At this point, we have algorithmically connected linear contextual bandit and linear bandit algorithms. Next, we bound the error due to the discrepancy between $\approxlcb$ and $\Omega$.

\subsection{Bounding the error due to the discrepancy between $\approxlcb$ and $\Omega$}\label{sec: bounding error}
From Section~\ref{sec: connection con}, we know how to \emph{execute} the linear contextual bandit algorithm by leveraging a linear bandit procedure in $\approxlcb$. However, there are errors due to the discrepancy between $\approxlcb$ and $\Omega$ which contribute to the final regret.
Let $\empD = \unif(\{\calA_1, \ldots, \calA_N\})$ denote the empirical distribution over the $N$ action sets drawn i.i.d.\ from $\calD$, and let $\approxlcb$ be constructed from them as in Lemma~\ref{lem:approxlcb-minksum}.

First, the loss estimator we constructed will be \emph{biased}. When
we sample a point $y_t\in\approxlcb$ and execute the corresponding
mixture policy $\pi_t$ such that
$|y_t-\mathbb{E}_{\calA\sim\empD}[\pi_t(\calA)]|\leq \frac{2}{\sqrt{N}}$, the expected loss
the learner observes is $\mathbb{E}_{\calA\sim\calD}[\inner{\pi_t(\calA),
\theta_t}]\neq \inner{y_t, \theta_t}$. This means that from the
viewpoint of the linear bandit problem on $\approxlcb$, the feedback is
\emph{misspecified}. This motivates us to develop a misspecification
robust linear bandit algorithm (Definition~\ref{def: misspec}) which
allows the feedback to not fully follow the standard linear bandit
protocol. We elaborate further in Section~\ref{sec: robust alg}. 

The second source of error comes from difference between the action sets $\approxlcb$ and $\Omega$. With the misspecification-robust linear bandit algorithm on $\approxlcb$, the learner has good regret bound on $\approxlcb$. However, the real regret we care about is on $\Omega$. This requires us to bound the difference between the two regret definitions: 
\begin{align*}
    \text{Real regret on $\Omega$ that we care about:}& \ \  \mathbb{E}_{\calA\sim \calD} [\inner{\pi_t(\calA), \theta_t}] - \mathbb{E}_{\calA\sim \calD} [\inner{\pi(\calA), \theta_t}], \\
    \text{Regret on $\approxlcb$ that the robust linear bandit algorithm can bound:}& \ \  \mathbb{E}_{\calA\sim \empD} [\inner{\pi_t(\calA), \theta_t}] - \mathbb{E}_{\calA\sim \empD} [\inner{\pi(\calA), \theta_t}]. 
\end{align*}


Both errors discussed above can be related to the difference
 $\sup_\pi\left|\mathbb{E}_{\calA\sim \empD}\left[\inner{\pi(\calA),
\theta_t}\right]-\mathbb{E}_{\calA\sim
\calD}\left[\inner{\pi(\calA), \theta_t}\right]\right|$. This can further be bounded
using Lemma~\ref{lem:uniform_convergence}, where we establish uniform
convergence over the set of all linear policies.  According to
Lemma~\ref{lem:uniform_convergence}, we have with probability at least
$1-\delta$, for all linear policies $\pi$, 
\begin{align}
    \left|\mathbb{E}_{\calA\sim \empD}\left[\inner{\pi(\calA), \theta_t}\right] - \mathbb{E}_{\calA\sim \calD}\left[\inner{\pi(\calA), \theta_t}\right] \right| \lesssim \sqrt{\frac{d\log(NK/\delta)}{N}}.  \label{eq: amount of mis}
\end{align}
This allows us to bound the two sources of errors mentioned above. 

\jack{I feel this could be a bit clearer. I'll make suggestions for alterations below momentarily.}



\subsection{Robust Linear Contextual Bandits}\label{sec: robust alg}
As discussed in Section~\ref{sec: bounding error}, we would like to
develop a linear bandit algorithm that tolerates misspecification.
Although there is a rich related literature, most prior results are
for the stochastic linear bandit problem and do not apply here. For
adversarial linear bandits with misspecification robustness, we are only
aware of the work by \cite{neu20} and \cite{liucorruption}. However, the
bound in \cite{neu20} has a worse $T^{2/3}$ regret, while the algorithms
of \cite{liucorruption} are either computationally inefficient or highly
sub-optimal.  Fortunately, our problem is slightly easier than that
studied by \cite{liucorruption}, as our learner has knowledge of the
amount of misspecification~$\epsilon$, and this amount remains the same in all rounds. This allows us to design the
computationally efficient Algorithm~\ref{alg: cew} with an improved
dependence on the amount of misspecification. 

Algorithm~\ref{alg: cew} is an adaptation of the clipped continuous
exponential weight algorithm of \cite{ito2020tight}: The $q_t$ in
Line~\ref{line: unclipped} of Algorithm~\ref{alg: cew} is the standard
continuous exponential weights, while the $\hatq_t$ in Line~\ref{line:
clipped} confines the support of $q_t$ within an ellipsoid centered
around the mean. This is helpful in obtaining a first-order bound
\citep{ito2020tight}. Sampling from $\hatq_t$ can be done with standard
techniques for sampling from a log-concave distribution (to sample from
$q_t$) plus rejection sampling (to correct the distribution to
$\hatq_t$), which can be done with polynomial calls to the linear optimization oracle
for $\approxlcb$. See \citet[Section~4.4]{ito2020tight} for a
discussion on the computational complexity. A linear optimization oracle
for $\approxlcb$ can be further reduced to linear optimization oracles
for individual action sets, as $\argmax_{v\in \approxlcb} \inner{v, \phi} = \frac{1}{N}\sum_{i=1}^N \argmax_{a_i\in\calA_i}\inner{a_i, \phi}$. Overall, continuous exponential weights over $\approxlcb$ can be conducted with $\poly(d,T)$ calls to the linear optimization oracles for individual action sets.


The key addition compared to \citep{ito2020tight} is the bonus term $b_t$ that ensures misspecification robustness. This bonus encourages additional exploration, preventing the learner from being misled by misspecified feedback. The form of the bonus for adversarial linear bandits was first developed in the series of work \citep{lee2020bias, zimmert2022return} aimed at  obtaining high-probability bounds. Our use of the bonus is similar to \cite{liu2024towards}, which tackles corruption and misspecification. In the regret analysis, this bonus term creates a negative regret term that offsets the regret overhead due to misspecification. 

The guarantee of
Algorithm~\ref{alg: cew} is given in the following theorem.  \tim{What is
the runtime of this algorithm?}
\cw{added}
\begin{theorem}\label{thm: main}
    Algorithm~\ref{alg: cew} is a $\sqrt{d}$-misspecification-robust
    linear bandit algorithm, as defined in Definition~\ref{def: misspec}. 
\end{theorem}

We remark that while there exist algorithms that are
$1$-misspecification-robust \citep{liucorruption}, their run time scales
at least with the number of actions. Algorithm \ref{alg: cew} achieves
$\alpha$-robustness with the smallest $\alpha$ we are aware of among
algorithms that run in $\poly(d,T)$ calls to the linear optimization oracle. 
\tim{Which
computational requirement is that?} \cw{changed the wording}

In fact, Algorithm~\ref{alg: cew} achieves an even more favorable small-loss regret bound, which can be leveraged to obtain a small-loss bound for linear contextual bandits when the simulator is available. We discuss this in Section~\ref{sec: small loss bound}. 

\begin{algorithm2e}[t]
    \tim{Shouldn't the distributions be over $\hat \Omega$ instead of
    $\A$?} \cw{changed}
    \caption{Misspecification-Robust Continuous Exponential Weights} \label{alg: cew}
    \textbf{Parameters}:  $\gamma=10\log (10dT)$, $\beta=T^{-4}$. \\
    \For{$t=1,2,\ldots, T$}{
         Define \label{line: unclipped}
         \begin{align*}
             q_t(y) = \frac{\exp\left(-\eta \sum_{\tau<t} \inner{ y, \hattheta_\tau - b_\tau}\right)}{\int_{\approxlcb} \exp\left(-\eta \sum_{\tau<t} \inner{ z, \hattheta_\tau - b_\tau} \right)\mathrm{d}z}, \ \  x_t = \mathbb{E}_{y\sim q_t}[y], \ \   \Sigma_t= \mathbb{E}_{y\sim q_t}[(y-x_t)(y-x_t)^\top].   
         \end{align*}\\
         Define \label{line: clipped}
         \begin{align*}
              \hatq_t(y) = \frac{q_t(y)\ind{\|y-x_t\|_{\Sigma_t^{-1}}^2 \leq d\gamma^2 }}{\int_{\approxlcb}q_t(z) \ind{\|z-x_t\|_{\Sigma_t^{-1}}^2 \leq d\gamma^2 }\dd z}, \ \  \hatSigma_t= \mathbb{E}_{y\sim \hatq_t}[(y-x_t)(y-x_t)^\top]. 
         \end{align*}\\
         Sample $y_t\sim \hatq_t$ and receive loss $c_t\in[0,1]$. \\[3pt]
         \tim{In text we are assuming losses in $[-1,+1]$, which is what
           we need for the worst-case rate, but we need
         losses in $[0,1]$ for first-order bound. Handle both cases.} \cw{For simplicity, I changed the losses in the whoe paper to [0,1]}
         Define $\hattheta_t =  (\beta I + \hatSigma_t)^{-1} (y_t-x_t) c_t$ and $b_t =  8\eta \left(\epsilon + \frac{1}{T^2}\right) \sum_{\tau<t}(\hattheta_\tau - b_\tau)$. 
    }
\end{algorithm2e}

\subsection{Combining Everything and Using the Doubling Trick}
Combining everything above, we are able to establish the regret bound for the linear contextual bandit problem. The proof of the following theorem is in Appendix~\ref{app: lcb bound}. 
\begin{theorem}\label{thm: main theorem cb}
    Algorithm~\ref{alg: lcb alg} with $\alg$ instantiated as a $\alpha$-misspecification-robust linear bandit algorithm ensures 
    \begin{align*}
        \mathbb{E}\left[\sum_{t=1}^T \inner{a_t, \theta_t} \right] &\leq \min_{\pi\in\Pi} \mathbb{E}\left[\sum_{t=1}^T \inner{\pi(\calA_t), \theta_t}\right] + \tilde{O}\left(d \sqrt{T} + \alpha Td\sqrt{\frac{\log(NKT)}{N}} \right). 
    \end{align*}
\end{theorem}


\begin{corollary}[Restatement of Theorem~\ref{thm: thm without }]\label{cor: context free}
    Given access to a $\alpha$-misspecification-robust adversarial linear bandit algorithm $\alg$, Algorithm~\ref{alg: lcb alg} with doubling trick achieves $\max_{\pi\in\Pi}\Reg{\pi}\leq \widetilde{O}(d\sqrt{T}+\alpha d\sqrt{T\log K})$ in adversarial linear contextual bandits without access to simulators. 
\end{corollary}
\begin{proof}
    We will use the doubling trick and restart Algorithm~\ref{alg: lcb alg} at times $2, 4, 8, 16, \ldots$, each time using the contexts received so far to estimate $\approxlcb$. Thus, in the 
$k$-th epoch, $\approxlcb$ is constructed by $N=\Theta(2^k)$ contexts, allowing us to bound the regret in epoch $k$ as 
\begin{align*}\tilde{O}\left(d\sqrt{2^k} + \alpha 2^k d\sqrt{\frac{\log(NKT)}{2^k}}\right)=\tilde{O}\left(d\sqrt{2^k} + \alpha d\sqrt{2^k \log(NKT)}\right)
\end{align*}
using Theorem~\ref{thm: main theorem cb}. Summing the regret over all epochs allows us to bound 
    \begin{align*}
        \max_{\pi\in\Pi}\Reg{\pi} \leq  \tilde{O}\left(\sum_{k=1}^{\log_2 T} \left(d\sqrt{2^k} + \alpha d\sqrt{2^k\log(NKT)} \right)\right) = \tilde{O}\left(d\sqrt{T}+\alpha d\sqrt{T\log K}\right). 
    \end{align*}
\end{proof}
By instantiating $\alg$ as Algorithm~\ref{alg: cew} (which is an $\sqrt{d}$-misspecification-robust algorithm by Theorem~\ref{thm: main}) and invoking Corollary~\ref{cor: context free}, we get the final regret bound as $\tilde{O}(\sqrt{d^3 T\log K})$. As we can assume $K\leq T^d$ without loss of generality (any action set can be discretized into no more than $T^d$ points and incurs a negligible regret of $d/T$ due to discretization error), the regret bound can be further improved to $\tilde{O}(\min\{d^2\sqrt{T}, \sqrt{d^3 T\log K}\})$.

\section{Small-Loss Bound with Access to Simulator}\label{sec: small loss bound}
The sub-optimal rate $d^2\sqrt{T}$ we obtained in Section~\ref{sec: linear bandit} comes from the misspecification
term $\alpha\sqrt{d}\epsilon T$ in the regret bound of robust linear bandits
(Definition~\ref{def: misspec}).  While it
is unclear how to further improve $\alpha$ or $\epsilon$, we demonstrate the power of
our reduction by further assuming access to simulator: it not only
allows us to recover the minimax optimal regret $d\sqrt{T}$ but also
allows us to obtain a first-order bound $d\sqrt{L^\star}$ when losses are non-negative, where $L^\star$ is the cumulative loss of the best policy. 

By the black-box nature of our reduction, what we additionally need is
just a misspecification-robust linear bandit algorithm with
\emph{small-loss} regret bound guarantee, formally defined as follows: 
\begin{definition}[$\alpha$-misspecification-robust adversarial linear bandit algorithm with small-loss bounds]\label{def: misspec smallloss}
A misspecification-robust linear bandit algorithm with small-loss bounds over action set $\approxlcb\subset \mathbb{R}^d$ has the following property: with a given 
$\epsilon>0$ and the guarantee that every time the learner chooses $y_t\in \approxlcb$, the loss received 
$c_t\in[0,1]$ satisfies $|\mathbb{E}_t[c_t] - \inner{y_t, \theta_t}| \leq \epsilon$, 
the algorithm ensures 
\begin{align}
   \mathbb{E}\left[\sum_{t=1}^T  \inner{y_t, \theta_t}\right] \leq \min_{y\in \approxlcb}\sum_{t=1}^T \inner{y, \theta_t} + \widetilde{O}\left( d\sqrt{\sum_{t=1}^T \inner{y,\theta_t} } + \alpha\sqrt{d}\epsilon T \right). \label{eq: LB small bound}
\end{align}
\end{definition}

The next theorem shows that Algorithm~\ref{alg: cew} satisfies Definition~\ref{def: misspec smallloss} with $\alpha=\sqrt{d}$.  
\begin{theorem}\label{thm: small loss oracle}
    Algorithm~\ref{alg: cew} is a $\sqrt{d}$-misspecification-robust linear bandit algorithm with small-loss bound defined in Definition~\ref{def: misspec smallloss}. 
\end{theorem}

With access to a misspecification-robust linear bandit algorithm with small-loss bounds, we have 
\tim{Need to restrict losses to be nonnegative for this to make sense}\cw{has changed the loss range to [0,1] throuout the paper}
\begin{theorem}\label{thm: thm with small loss}
    Given access to simulator that can generate free contexts, and access to an $\alpha$-misspecification-robust adversarial linear bandit algorithm with small-loss regret bound guarantee, we can achieve $\min_{\pi\in\Pi}\E{\Reg{\pi}}\leq \widetilde{O}(d\sqrt{L^\star})$ in adversarial linear contextual bandits, where 
    \begin{align*}
        L^\star = \min_{\pi\in\Pi}\E{\sum_{t=1}^T \inner{\pi(\calA_t), \theta_t}}
    \end{align*}
    is the expected total loss of the best policy. This is achieved with $O(\alpha^2 d^2T^2)$ calls to the simulator. 
\end{theorem}
The proof of Theorem~\ref{thm: thm with small loss} is very similar to
Theorem~\ref{thm: main theorem cb}, except that now, with access to the simulator, we are able to make $N$ in Theorem~\ref{thm: main theorem cb} large enough that the second term in Theorem~\ref{thm: main theorem cb} is negligible. We provide the omitted proofs in Appendix~\ref{sec: small loss app}. 

\section{Conclusion and Open Questions}
We have provided a general framework that reduces adversarial linear
contextual bandits to misspecification-robust linear bandits in a
black-box manner. It achieves $\tOO(d^{2}\sqrt{T})$ regret without a simulator, and is the first algorithm we know of that handles combinatorial bandits with stochastic action sets and adversarial losses efficiently.  
The requirement of misspecification robustness stems from our need to
use an approximate feasible set $\approxlcb$ because we do not have
direct access to the exact feasible set $\Omega$, which depends on the
action set distribution $\calD$. 




Two open questions are left by our work: 
First, can the regret be further improved to the near-optimal
$\tOO(d\sqrt{T})$ bound with a polynomial-time algorithm without simulators? Second, can we achieve $\poly(d)\sqrt{L^\star}$ regret without simulator? For the first question, one may try to generalize the approach of \cite{schneider2023optimal}. For the second question, one idea is to establish a Bernstein-type counterpart of Lemma~\ref{lem:uniform_convergence}, potentially drawing ideas from \cite{bartlett2005local, liang2015learning}. 

We expect that our approach has wider applications than adversarial linear contextual bandits. For example, our approach may be generalized to linear MDPs with fixed transition and adversarial losses \citep{luo2021policy, sherman2023improved, dai2023refined, kongimproved, liu2024towards} and facilitate learning with exponentially large or continuous action sets. 

\section*{Acknowledgements}

Mayo and Van Erven were supported by the Netherlands Organization for
Scientific Research (NWO) under grant number VI.Vidi.192.095.

\bibliographystyle{apalike}
\bibliography{lcb}


\appendix

\section{Comparison with \cite{hanna2023contexts}} \label{app: comparison}

Our work and \cite{hanna2023contexts} both reduce linear contextual bandits to linear bandits with misspecification. The linear bandits the two works reduce to are slightly different. For comparison, let $\calA_1, \ldots, \calA_N$ be empirical action sets drawn from~$\calD$, and define 
\begin{align}
    \approxlcb = \left\{\hat\Psi(\pi)\Big{|} \pi \in \Pi\right\}, \qquad \approxlcb_\lin = \left\{\hat\Psi(\pi)\Big{|} \pi \in \Pi_\lin\right\}, \qquad \approxlcb_\lin^\epsilon = \left\{\hat\Psi(\pi)\Big{|} \pi \in \Pi_\lin^\epsilon\right\}, \label{eq: thre set} 
\end{align}
where $\hat{\Psi}(\pi)=\frac{1}{N}\sum_{i=1}^N \pi(\calA_i)$, $\Pi_\lin$ is the set of linear policies defined in \eqref{eq:linclass}, and 
\begin{align*}
    \Pi_{\lin}^\epsilon =\left\{\pi_{\phi}~\big|~ \phi\in  \text{$\epsilon$-net of the unit ball in $\mathbb{R}^d$}\right\}.    
\end{align*}
It holds that $\approxlcb_\lin^\epsilon\subset\approxlcb_\lin \subset \approxlcb$ because $\Pi^\epsilon_\lin \subset \Pi_\lin \subset \Pi$. The relation among the three sets in \eqref{eq: thre set} is the following: $\approxlcb_\lin$ lies on the boundary of $\approxlcb$, and $\approxlcb_\lin^\epsilon$ is a discretized subset of $\approxlcb_\lin$. It has been shown that there must exist an optimal policy in $\Pi_\lin$ (see our arguments in Section~\ref{sec:setting}), and that $\Pi_\lin^\epsilon$ must contain an $\epsilon$-near-optimal policy (see~(47) of \cite{liu2024bypassing}). 

Our work reduces the linear contextual bandit problem to linear bandits in $\approxlcb$, while \cite{hanna2023contexts} reduces it to linear bandits in $\approxlcb^\epsilon_\lin$. According to the discussion above, both suffice to establish no regret guarantees against the optimal policy. The key difference lies in the computational efficiency: $\approxlcb$ is a convex set whose linear optimization oracle can be reduced to linear optimization oracles of individual action sets $\calA_1, \ldots, \calA_N$, which usually admit polynomial time implementation, while in \cite{hanna2023contexts}, $\approxlcb^\epsilon_{\lin}$ is a set containing $(\frac{1}{\epsilon})^{\Theta(d)}$ discrete points with $\epsilon$ chosen as $\Theta(\frac{1}{T})$. It remains open whether $\approxlcb^\epsilon_\lin$ can be represented efficiently and whether \cite{hanna2023contexts}'s algorithm can be implemented in polynomial time. 
In addition, our work operates on the adversarial-loss setting, strictly generalizing \cite{hanna2023contexts}'s stochastic-loss setting. 

A technical question related to the discussion above is whether we can apply the discretization idea from \cite{hanna2023contexts} purely in the analysis but not in the algorithm, i.e., bounding the regret overhead by the discretization error without restricting the learner's policy to the discretized linear policies. This would allow us to avoid the Rademacher complexity analysis (Lemma~\ref{lem:uniform_convergence}) and instead use Hoeffding’s or Bernstein’s inequality combined with a union bound, possibly leading to an improved regret bound.
However, there is a key technical challenge: the mapping $\hat{\Psi}(\pi_\phi)$ is not continuous in $\phi$, which prevents us from bounding $\|\hat{\Psi}(\pi_\phi) - \hat{\Psi}(\pi_{\phi'})\|$ by a constant times $\|\phi - \phi'\|$, where $\phi$ is an arbitrary point and $\phi'$ is a point in the discretized set. This would leave us no control on the performance of policies outside the set of discretized linear policies $\Pi^\epsilon_\lin$. As a result, we are unable to derive a uniform convergence result analogous to Lemma~\ref{lem:uniform_convergence} using this approach.

\section{Proof of Lemma~\ref{lem:uniform_convergence}}
\label{app:proof_uniform_convergence}

\uniformconvergence*

\begin{proof}
  We will first reduce the result to a statement about uniform
  convergence for linear multiclass classifiers with $K$ classes and an
  unusual loss function. To this end, note first that we can assume
  without loss of generality that $\calA = \{V_1, \ldots, V_K\}$, where
  the $V_1,\ldots,V_K$ are random vectors in~$\reals^d$. (If $\calA$ has
  fewer than $K$ elements, then consider it as a multiset and add
  repetitions of one of its elements.) We can assume all randomness in
  $\calA$ is determined by an underlying random variable $Z$ and that
  $V_y = -g(Z,y)$ for each `class' $y \in [K]$, where $g$ is a
  class-sensitive feature map in the sense of
  \citet[Section~17.2]{ShalevShwartzBenDavid2014}. Defining the
  linear multiclass classifier
  \[
    c_\phi(Z) = \argmax_{y \in [K]} \langle g(Z,y), \phi \rangle,
  \]
  we then obtain the correspondence
  \[
    \pi_\phi(\calA) = V_{c_\phi(Z)}.
  \]
  For any multiclass classifier $c$, let $f(Z,c) = \langle V_{c(Z)},
  \theta \rangle$ be our `loss function'. Then
  \begin{align*}
    \big\langle \Psi(\pi_\phi), \theta \big\rangle&=
    \Exp_Z [f(Z,c_\phi)],
    &
    \big\langle \hat
    \Psi(\pi_\phi),\theta\big\rangle
    &= \frac{1}{N} \sum_{i=1}^N f(Z_i,c_\phi).
  \end{align*}
  We therefore need to show the following uniform convergence result,
  with probability at least $1-\delta$,
  \begin{equation}\label{eqn:multiclass_uniform_convergence}
    \sup_{c \in \Cclass} \Big|\Exp[f(Z,c)] - \frac{1}{N} \sum_{i=1}^N
    f(Z_i,c)\Big|
    \leq 
      2 b\sqrt{\frac{2 d\ln (NK^2)}{N}}
      + b\sqrt{\frac{2 \ln(4/\delta)}{N}}
  \end{equation}
  for the class of linear multiclass classifiers 
  \[
    \Cclass = \Big\{ c_\phi \mid \phi \in \reals^d\Big\},
  \]
  with loss function~$f$.
  In order to establish \eqref{eqn:multiclass_uniform_convergence}, let
  $S = (Z_1,\ldots,Z_N)$, and consider the empirical Rademacher
  complexity
  \[
     \Rad(\Cclass, S)
     = \frac{1}{N} \Exp_{\sigma_1,\ldots,\sigma_N}
     \Big[ \sup_{c \in \Cclass} \sum_{i=1}^N \sigma_i f(Z_i,c)
     \Big],
  \]
  where the $\sigma_i$ are independent Rademacher random variables with
  $\Pr(\sigma_i = -1) = \Pr(\sigma_i = +1) = 1/2$. Since $|f(Z,c)| \leq
  b$ for $c \in \Cclass$ by assumption, standard concentration bounds in
  terms of Rademacher complexity imply that
  \begin{equation}\label{eqn:rad_control}
    \sup_{c \in \Cclass} \Big|\Exp_Z[f(Z,c)] - \frac{1}{N} \sum_{i=1}^N
    f(Z_i,c)\Big|
    \leq 2\Exp_{S'}[\Rad(\Cclass, S')] + b \sqrt{\frac{2 \ln(4/\delta)}{N}}
  \end{equation}
  with probability at least $1-\delta$. (This follows, for instance, by
  observing that
  \begin{multline*}
    \sup_{c \in \Cclass} |\Exp_Z[f(Z,c)] - \frac{1}{N} \sum_{i=1}^N
    f(Z_i,c)|\\
    = \max \Big\{\sup_{c \in \Cclass} \Exp_Z[f(Z,c)] - \frac{1}{N} \sum_{i=1}^N
    f(Z_i,c), \sup_{c \in \Cclass} \Exp_Z[-f(Z,c)] - \frac{1}{N} \sum_{i=1}^N
  (-f(Z_i,c))\Big\}
  \end{multline*}
  and applying part 1 of Theorem~26.5 of
  \citet{ShalevShwartzBenDavid2014} separately to $f$ and $-f$ to
  control both parts in the maximum separately using a union bound; then
  noting that $f$ and~$-f$ have the same Rademacher complexity.)

  We proceed to bound the Rademacher complexity on the right-hand side
  of \eqref{eqn:rad_control}. First, let $\Cclass_S =
  \{(c(Z_1),\ldots,c(Z_N)) \mid c \in \Cclass\}$ denote the restriction
  of $\Cclass$ to the sample $S$. Then
  \[
    \Rad(\Cclass, S)
      = \Rad(\Cclass_S, S)
      \leq b\sqrt{\frac{2 \ln |\Cclass_S|}{N}}
  \]
  by Massart's lemma \citep{ShalevShwartzBenDavid2014}. As discussed by
  \citet[Chapter~29]{ShalevShwartzBenDavid2014}, one possible
  generalization of the Vapnik-Chervonenkis dimension to multiclass
  classification is Natarajan's dimension $\Ndim(\Cclass)$. Natarajan's
  lemma \citep[p.\,93]{Natarajan1989},
  \citep[Lemma~29.4]{ShalevShwartzBenDavid2014} shows that
  \[
    |\Cclass_S| \leq N^{\Ndim(\Cclass)} K^{2\Ndim(\Cclass)},
  \]
  and, for linear multiclass classifiers, it is also known
  \citep[Theorem~29.7]{ShalevShwartzBenDavid2014} that
  \[
    \Ndim(\Cclass) \leq d.
  \]
  Putting all inequalities together, it follows that
  \[
    \sup_{c \in \Cclass} \Big|\Exp_Z[f(Z,c)] - \frac{1}{N} \sum_{i=1}^N
    f(Z_i,c)\Big|
    \leq
      2 b\sqrt{\frac{2 d\ln (NK^2)}{N}}
      + b\sqrt{\frac{2 \ln(4/\delta)}{N}}
  \]
  with probability at least $1-\delta$, as required.
\end{proof}

\section{Omitted Details in Section~\ref{sec: linear bandit}}

\subsection{Proof of Lemma~\ref{lem:approxlcb-minksum}}\label{app:approxlcb-minksum}

\approxlcbminksum*
\begin{proof}
Group indices by action set: for each distinct $\A$ among $\calA_1,
\ldots, \calA_N$, let $I_\A = \{i : \calA_i = \A\}$. The constraint $a_i
= a_j$ for $i,j \in I_\A$ restricts the contribution of that group to
$\{|I_\A| \cdot a : a \in \conv(\A)\} = |I_\A| \cdot \conv(\A)$. Without
the constraint, the contribution is the Minkowski sum of $|I_\A|$ copies
of $\conv(\A)$, which also equals $|I_\A| \cdot \conv(\A)$ by repeated
application of the identity $\lambda \A' + \mu \A' = (\lambda+\mu)\A'$
for any convex set $\A'$ and $\lambda,\mu\geq 0$
\citep[Remark~1.1.1]{Schneider2014}. Since the two contributions coincide
for every group, the overall sets are equal.
\end{proof}

\subsection{Proof of Lemma~\ref{lem: Frankwolfe}}\label{app: Frankwolfe}
\frankwolfe*
\begin{proof}[Proof of Lemma~\ref{lem: Frankwolfe}]
    By the definitions of $v_j$ and $z_j$ in Algorithm~\ref{alg: decompose}, we have 
    \begin{align}
       z_M = \frac{1}{M}\sum_{j=1}^M v_j =
       \frac{1}{N}\frac{1}{M}\sum_{i=1}^N\sum_{j=1}^M
       \pi_{\phi_j}(\calA_i) = \frac{1}{M} \sum_{j=1}^M \hat{\Psi}(\pi_{\phi_j}). \label{eq: notice2} 
    \end{align}
    Below, we bound $\|z_m- y\|$ for any $m\geq 1$. 
    By the algorithm, we have 
    \begin{align}
        v_j\in \frac{1}{N}\sum_{i=1}^N \argmin_{z\in\calA_i} \inner{z,\phi_j} \in \argmin_{v\in \approxlcb} \inner{v, \phi_j}. \label{eq: optimality}
    \end{align}
    Therefore, for any $j \geq 2$, 
    \begin{align*}
        \|z_j - y\|^2 
        &=  \left\|\left(1-\frac{1}{j}\right) z_{j-1} + \frac{1}{j}v_j - y\right\|^2 \\
        &=\|z_{j-1} - y\|^2 - \frac{2}{j}\inner{z_{j-1} - y, z_{j-1} - v_j} + \frac{1}{j^2}\|z_{j-1}-v_j\|^2 \\
        &\leq \|z_{j-1} - y\|^2 - \frac{2}{j}\inner{\phi_j, z_{j-1} - v_j} + \frac{4}{j^2} \tag{$\calA_i$ is in unit ball}\\
        &\leq \|z_{j-1} - y\|^2 - \frac{2}{j}\inner{\phi_j, z_{j-1} - y} + \frac{4}{j^2}   \tag{by \eqref{eq: optimality}} \\
        &= \|z_{j-1} - y\|^2 - \frac{2}{j}\|z_{j-1} - y\|^2+ \frac{4}{j^2}.    
    \end{align*}
    By induction over $j$, one can show $\|z_j - y\|^2 \leq \frac{4}{j}$. Letting $j=M$ and using \eqref{eq: notice2} proves the desired inequality. 
\end{proof}

\subsection{Robust Linear Bandits}\label{app: robust linear }
\begin{lemma}[Lemma 14 of \cite{zimmert2022return}] \label{lem: useful} Let $F$ be a $\nu$-self-concordant barrier for $\calA\subset \mathbb{R}^d$ for some $\nu\geq 1$.  Then for any $x, u\in \calA$, 
\begin{align*}
    \|x-u\|_{\nabla^2 F(x)} \leq \gamma'\langle x-u, \nabla F(x)\rangle + 6\gamma'\nu 
\end{align*}
where $\gamma' = \frac{8}{3\sqrt{3}} + \frac{7^{\frac{3}{2}}}{6\sqrt{3\nu}}$ ($\gamma'\in[1,4]$ for $\nu\geq 1$). 
\end{lemma}
It is known that the continuous exponential weight algorithm is equivalent to FTRL with entropic barrier as the regularizer  together with a particular sampling scheme \citep{bubeck2015entropic,zimmert2022return}. We summarize the equivalence in the following lemma, the details of which can be found in \cite{zimmert2022return}.   
\begin{lemma}[Facts from \cite{bubeck2015entropic}, \cite{zimmert2022return}]\label{lem: equivalen}
Consider Algorithm~\ref{alg: cew}. Let $x_t = \mathbb{E}_{y\sim q_t}[y]$ and let $F: \calA\to \mathbb{R}$ be the entropic barrier on $\calA$. Then we have 
\begin{align*}
    x_t = \argmin_{y\in \approxlcb}  \inner{y, \sum_{\tau<t} \left(\hattheta_\tau - b_\tau\right)} + \frac{F(y)}{\eta}.   
\end{align*}
Furthermore, 
\begin{align*}
    \nabla F(x_t) = -\eta \sum_{\tau<t} (\hattheta_\tau - b_\tau) \quad \text{and} \quad \nabla^2 F(x_t) 
 = \left(\mathbb{E}_{y\sim q_t}[ (y-x_t)(y-x_t)^\top]\right)^{-1}. 
\end{align*}
\end{lemma}

\begin{lemma}[Lemma 4, \cite{ito2020tight}]\label{lem: closeness}
    Let $q_t, \hatq_t, \Sigma_t, \hatSigma_t$ be as defined in Algorithm~\ref{alg: cew}. Then for any $f(y): \calA\to [-1,1]$, 
    \begin{align*}
        \left|\mathbb{E}_{y\sim q_t} [f(y)] - \mathbb{E}_{y\sim \hatq_t} [f(y)]\right| \leq 10d \exp(-\gamma) \leq \frac{1}{d^5T^{10}}. 
    \end{align*}
    Furthermore, 
    \begin{align*}
        \frac{3}{4} \Sigma_t \preceq \hatSigma_t \preceq \frac{4}{3}\Sigma_t. 
    \end{align*}
\end{lemma}


\begin{proof}[Proof of Theorem~\ref{thm: main}]
    First, we decompose the regret as follows:  
    \begin{align*}
        &\mathbb{E}\left[\sum_{t=1}^T \inner{y_t - u, \theta_t}\right]\\
         &\leq \mathbb{E}\left[\sum_{t=1}^T \inner{\mathbb{E}_{y\sim q_t}[y] - u, \theta_t}\right] + O(1)\\ 
        &= \underbrace{\mathbb{E}\left[\sum_{t=1}^T  \inner{x_t - u, \hattheta_t - b_t}\right]}_{\ftrl} + \underbrace{\mathbb{E}\left[\sum_{t=1}^T \inner{x_t - u, \theta_t - \hattheta_t}\right]}_{\bias} + \underbrace{\mathbb{E}\left[\sum_{t=1}^T 
\inner{x_t-u, b_t} \right]}_{\bonus} + O(1). 
    \end{align*}
    \paragraph{Bounding \ftrl term } The analysis for the \ftrl term follows that in \cite{ito2020tight}. Specifically, directly following their Lemma~5, Lemma~6, and the analysis below Lemma~6, we have the following: as long as $\eta \|\hattheta_t - b_t\|_{\Sigma_t}\leq 1$, we have 
    \begin{align*}
       \ftrl 
       &\leq \frac{d\log T}{\eta} + \eta \mathbb{E}\left[\sum_{t=1}^T \|\hattheta_t - b_t\|_{\Sigma_t}^2\right] + O(1) \\
       &\leq \frac{d\log T}{\eta} + 2\eta \mathbb{E}\left[\sum_{t=1}^T  \|\hattheta_t\|_{\Sigma_t}^2 + \|b_t\|_{\nabla^{-2} F(x_t)}^2\right] + O(1).   \tag{by Lemma~\ref{lem: equivalen}}
    \end{align*}
    For the two middle terms, we have 
    \begin{align*}
        \|\hattheta_t\|_{\Sigma_t}^2 
        &\leq (y_t - x_t)^\top \left(\beta I + \hatSigma_t\right)^{-1}\Sigma_t \left(\beta I + \hatSigma_t\right)^{-1} (y_t-x_t) c_t^2\\
        &\leq 2(y_t - x_t)^\top \Sigma_t^{-1}(y_t-x_t) c_t^2 \tag{by Lemma~\ref{lem: closeness}}  \\
        &\leq 2d\gamma^2c_t^2,    \tag{by the truncation in the algorithm}
    \end{align*}
    and 
    \begin{align*}
        \|b_t\|^2_{\nabla F(x_t)} 
        &= \left(2\sqrt{d}\gamma \epsilon + \frac{1}{T^2}\right)^2 \|\nabla F(x_t)\|_{\nabla^{-2} F(x_t)}^2 \leq O\left(d^2\gamma^2 \epsilon^2 + \frac{1}{T^2}\right), 
    \end{align*}
    where we used the fact that $F(\cdot)$ is a $O(d)$ self-concordant barrier and thus $\|\nabla F(x_t)\|_{\nabla^{-2} F(x_t)}^2\leq O(d)$. Thus, since $\eta \leq \frac{1}{d\gamma^2}$, we have 
    \begin{align*}
        \ftrl &\leq \frac{d\log T}{\eta} + O\left(\eta d\gamma^2 \sum_{t=1}^T c_t^2  + \eta T  d^2\gamma^2 \epsilon^2\right). 
    \end{align*}

    \paragraph{Bounding \bias term} Let $\mathbb{E}[c_t] = y_t^\top \theta_t + \epsilon_t(y_t)$, where $\epsilon_t(y)$ is the amount of misspecification when choosing $y$. By assumption, we have $|\epsilon_t(y)|\leq \epsilon$ for any $y$. 
\begin{align*}
   \mathbb{E}_t[\hattheta_t] &= \mathbb{E}_t\left[ (\beta I + \hatSigma_t)^{-1} (y_t - x_t) \left(y_t^\top \theta_t + \epsilon_t(y_t) \right) \right] \\
   &= \mathbb{E}_t\left[ (\beta I + \hatSigma_t)^{-1} (y_t - x_t) \left((y_t-x_t)^\top \theta_t + \epsilon_t(y_t) \right) \right] + \mathbb{E}_t\left[(\beta I + \hatSigma_t)^{-1} (y_t - x_t) x_t^\top \theta_t \right] \\
   &= \theta_t - \beta (\beta I + \hatSigma_t)^{-1}\theta_t + \mathbb{E}_t\left[(\beta I + \hatSigma_t)^{-1} 
(y_t-x_t) \epsilon_t(y_t)\right] + (\beta I + \hatSigma_t)^{-1} (\hatx_t
- x_t) x_t^\top \theta_t.
\end{align*}
Using this we get 
\begin{align}
    &\left|\inner{ x_t - u, \theta_t - \mathbb{E}_t[\hattheta_t]}\right|  \nonumber  \\
    &\leq \beta \left|(x_t-u)^\top (\beta I + \hatSigma_t)^{-1} \theta_t\right| + \underbrace{\left| (x_t-u)^\top \mathbb{E}_t\left[(\beta I + \hatSigma_t)^{-1} 
(y_t-x_t)\epsilon_t(y_t) \right] \right|}_{(\star)}  \nonumber  \\
    &\qquad + \left|(x_t-u)^\top (\beta I + \hatSigma_t)^{-1} (\hatx_t - x_t) x_t^\top \theta_t\right|. \label{eq: to continue 3}
    \end{align}
    We handle $(\star)$ as follows: 
\begin{align*}
    (\star) 
    &= \mathbb{E}_t\left[\sqrt{(x_t-u)^\top (\beta I + \hatSigma_t)^{-1} (y_t-x_t)(y_t-x_t)^\top (\beta I + \hatSigma_t)^{-1} (x_t-u) \epsilon_t(y_t)^2}\right] \\
    &\leq \sqrt{(x_t-u)^\top (\beta I + \hatSigma_t)^{-1} \mathbb{E}_t\left[\epsilon_t(y_t)^2(y_t-x_t)(y_t-x_t)^\top\right] (\beta I + \hatSigma_t)^{-1} (x_t-u) } \\
    &\leq \epsilon \sqrt{(x_t-u)^\top (\beta I + \hatSigma_t)^{-1} \hatSigma_t (\beta I + \hatSigma_t)^{-1} (x_t-u) } \\
    &\leq \epsilon\|x_t-u\|_{(\beta I + \hatSigma_t)^{-1}}. 
\end{align*}
Continuing from \eqref{eq: to continue 3}, we get 
\begin{align*}
    &\left|\inner{ x_t - u, \theta_t - \mathbb{E}_t[\hattheta_t]}\right| \\
    &\leq \beta \|x_t-u\|_{(\beta I + \hatSigma_t)^{-1} }\|\theta_t\|_{(\beta I + \hatSigma_t)^{-1}} + \epsilon\|x_t-u\|_{(\beta I + \hatSigma_t)^{-1}}  + \|x_t-u\|_{(\beta I + \hatSigma_t)^{-1}} \|\hatx_t-x_t\|_{(\beta I + \hatSigma_t)^{-1}} \\
    &\leq  \sqrt{d\beta } \|x_t-u\|_{(\beta I + \hatSigma_t)^{-1}} + \epsilon \|x_t-u\|_{(\beta I + \hatSigma_t)^{-1}}  + \sqrt{\frac{1}{\beta}} \|x_t-u\|_{(\beta I + \hatSigma_t)^{-1}} \|\hatx_t - x_t\|_2, 
\end{align*}
    where in the last inequality we use that $\|y_t-x_t\|_{(\beta I + \hatSigma_t)^{-1}} \leq 2\|y_t-x_t\|_{\Sigma_t^{-1}} $ by 
Lemma~\ref{lem: closeness} and the assumption that $\|\theta_t\|_2 \leq \sqrt{d}$.  

    By Lemma~\ref{lem: closeness} we have $\|x_t-\hatx_t\| = \|\mathbb{E}_{y\sim q_t}[y] - \mathbb{E}_{y\sim \hatq_t}[y]\| \leq \sqrt{d}d^{-5}T^{-10}$. By the choice of $\beta = d^{-2}T^{-4}$, we can bound the expectation of the sum of the last expression as 
    \begin{align*}
        \bias 
        &\leq \mathbb{E}\left[\sum_{t=1}^T \|x_t-u\|_{(\beta I +
        \hatSigma_t)^{-1}} \left(\epsilon + \frac{1}{T^2}\right)\right].
\end{align*}
\paragraph{Bounding Bonus term} Notice that with the equivalence established in Lemma~\ref{lem: equivalen}, our bonus term $b_t$ can also be written as 
\begin{align*}
    b_t = -8\left(\epsilon + \frac{1}{T^2}\right)\nabla F(x_t),
\end{align*}
where $F(\cdot)$ is the entropic barrier on $\calA$.  Using
Lemma~\ref{lem: useful} and the fact that $F(\cdot)$ is an $O(d)$-self-concordant barrier, we can bound
\begin{align*}
   \inner{x_t-u,b_t} 
   &\leq -\frac{(8\epsilon + 8/T^2)}{4}\|x_t-u\|_{\nabla^2 F(x_t)} + 6\epsilon\nu  \\
   &= -\left(2\epsilon + \frac{2}{T^2}\right)\|x_t-u\|_{\nabla^2 F(x_t)} + O(d\epsilon) \\
   &= -\left(2\epsilon + \frac{2}{T^2}\right)\|x_t-u\|_{\Sigma_t^{-1}} + O(d\epsilon) \\
   &\leq -\left(\epsilon +
   \frac{1}{T^2}\right)\|x_t-u\|_{\hatSigma_t^{-1}} + O(d\epsilon).    \tag{by Lemma~\ref{lem: closeness}}
\end{align*}
Thus, 
\begin{align*}
    \bonus \leq \mathbb{E}\left[-\sum_{t=1}^T \left(\epsilon + \frac{1}{T^2}\right)\|x_t-u\|_{\hatSigma_t^{-1}} + O\left(dT\epsilon\right) \right]. 
\end{align*}
\paragraph{Adding up all terms} 
Adding up the three terms, we get 
\begin{align*}
    \mathbb{E}\left[\sum_{t=1}^T \inner{y_t - u, \theta_t}\right] \leq \frac{d\log T}{\eta} + O\left(\eta d\gamma^2 \mathbb{E}\left[\sum_{t=1}^T c_t^2\right] + dT\epsilon\right). 
\end{align*}
By the assumption $c_t\in[0,1]$ and the assumption $\left|\mathbb{E}_t\left[ c_t \right] - y_t^\top \theta_t\right| \leq \epsilon$, we can further bound the right-hand side by 
\begin{align*}
   &O\left(\frac{d\log T}{\eta} + \eta d\gamma^2 \mathbb{E}\left[\sum_{t=1}^T c_t\right] + dT\epsilon\right) \leq O\left(\frac{d\log T}{\eta} + \eta d\gamma^2 \mathbb{E}\left[\sum_{t=1}^T \inner{y_t, \theta_t}\right] + (d+\eta d\gamma^2)T\epsilon\right).
\end{align*}
Then, by rearranging, we find that 
\begin{align*}
    \mathbb{E}\left[\sum_{t=1}^T \inner{y_t - u, \theta_t}\right]  \leq O\left(\frac{d\log T}{\eta} + \eta d\gamma^2 \sum_{t=1}^T \inner{u,\theta_t} + dT\epsilon\right). 
\end{align*}
Choosing the optimal $\eta$, we further bound 
\begin{align}
    \mathbb{E}\left[\sum_{t=1}^T \inner{y_t - u, \theta_t}\right] 
 = O\left(d\gamma\sqrt{(\log T)\sum_{t=1}^T \inner{u,\theta_t} } +
 dT\epsilon \right) = \tilde{O}\left(d\sqrt{T} + dT\epsilon\right).
 \label{eq: pseudosmall loss bound}
\end{align}
By Definition~\ref{def: misspec}, this is a $\sqrt{d}$-misspecification-robust algorithm.

\end{proof}

\subsection{Regret bound of LCB}\label{app: lcb bound}

\begin{proof}[Proof of Theorem~\ref{thm: main theorem cb}] 
Let $\hatPi$ be the set of linear policies created from the vertices of $\approxlcb$, and let $\Pi$ be set of all linear policies. 

The regret bound guaranteed by the $\alpha$-misspecification-robust linear bandit problem is 
\begin{align}
    \mathbb{E}\left[\sum_{t=1}^T \inner{y_t, \theta_t} \right] &\leq \min_{y\in\approxlcb} \sum_{t=1}^{T}   y^\top \theta_t + \tilde{O}\left(d\sqrt{T} + \alpha \sqrt{d}\epsilon T\right)
    \label{eq: CEW guarantee}  
\end{align}
for some $\alpha\geq 1$. 
By Lemma~\ref{lem: Frankwolfe} with the choice of $M=N$, we have 
\begin{align}
    \left\|y_t - \mathbb{E}_{\calA\sim \empD}\left[ 
\pi_t(\calA)  \right]\right\|\leq \frac{2}{\sqrt{N}}.\label{eq: error222} 
\end{align}
We further define
\begin{align*}
    z_t = \mathbb{E}_{\calA\sim \calD}\left[ 
\pi_t(\calA)  \right], \qquad z^\star = \mathbb{E}_{\calA\sim \calD}\left[ 
\pi^\star(\calA)  \right], \qquad y^\star = \mathbb{E}_{\calA\sim \empD}\left[\pi^\star(\calA)\right], 
\end{align*}
where $\pi^\star\in\Pi$ is the final regret comparator. 
Define 
\begin{align*}
    \epsilon = 6\sqrt{\frac{d\log(N KT/\delta)}{N}}, 
\end{align*}
where $N$ is the number of contexts used to construct $\approxlcb$. 
Then by Lemma~\ref{lem:uniform_convergence} and \eqref{eq: error222}, we have $\left|\inner{y_t,\theta_t} - \inner{z_t, \theta_t}\right|\leq \epsilon$ and $\left|\inner{y^\star,\theta_t} - \inner{z^\star, \theta_t}\right|\leq \epsilon$ with probability at least $1-\delta$ for all $t$.  Choosing $\delta=\frac{1}{T^2}$, we obtain 

\begin{align*}
    \mathbb{E}\left[\sum_{t=1}^T \inner{z_t,\theta_t} \right] 
    &\leq \mathbb{E}\left[\sum_{t=1}^T \inner{y_t,\theta_t} \right] + T\epsilon \\
    &\leq \sum_{t=1}^T \inner{y^\star, \theta_t} + \tilde{O}\left(d \sqrt{T} + \alpha \sqrt{d}T\epsilon\right)  \tag{by \eqref{eq: CEW guarantee}}\\
    &\leq \sum_{t=1}^T \inner{z^\star, \theta_t} + \tilde{O}\left(d \sqrt{T} + \alpha \sqrt{d} T\epsilon \right) \\
    &= \sum_{t=1}^T \inner{z^\star, \theta_t} + \tilde{O}\left(d \sqrt{T} + \alpha d T \sqrt{\frac{\log(NKT)}{N}} \right).
\end{align*}
This proves the theorem. 
\end{proof}

\section{Omitted Details in Section~\ref{sec: small loss bound}}
\label{sec: small loss app}

\begin{proof}[Proof of Theorem~\ref{thm: small loss oracle}]

This is by the same proof as Theorem~\ref{thm: main}, just noticing that
it actually achieves a small-loss bound in \eqref{eq: pseudosmall loss bound}. 

\end{proof}

\begin{proof}[Proof of Theorem~\ref{thm: thm with small loss}] 

This is similar to the proof of Theorem~\ref{thm: main theorem cb}, except that we replace \eqref{eq: CEW guarantee} by 
\begin{align}
    \mathbb{E}\left[\sum_{t=1}^T \inner{y_t, \theta_t} \right] &\leq \min_{y\in\approxlcb} \sum_{t=1}^{T}   y^\top \theta_t + \tilde{O}\left(d\sqrt{\sum_{t=1}^T \inner{y, \theta_t} } + \alpha \sqrt{d}\epsilon T\right). 
    \label{eq: CEW guarantee small loss}  
\end{align}
Following the same steps, we get 
\begin{align}
    \mathbb{E}\left[\sum_{t=1}^T \inner{z_t,\theta_t} \right] 
    \leq \sum_{t=1}^T \inner{z^\star, \theta_t} + \tilde{O}\left(d
    \sqrt{\sum_{t=1}^T \inner{z^\star, \theta_t}} + \alpha d T
  \sqrt{\frac{\log(NKT)}{N}} \right), \label{eq: small loss bound} 
\end{align}
where $z_t = \mathbb{E}_{\calA\sim \calD}\left[ 
\pi_t(\calA)  \right]$ and $z^\star = \mathbb{E}_{\calA\sim \calD}\left[ 
\pi^\star(\calA)  \right]$. 

Since the learner is given simulator access, she can draw $N=\tilde{\Omega}(\alpha^2 d^2 T^2)$ samples and make the last term in \eqref{eq: small loss bound} be a constant. This will give a final regret bound of $\tilde{O}\Big(d \sqrt{\sum_{t=1}^T \inner{z^\star, \theta_t}}\Big)=\tilde{O}(d\sqrt{L^\star})$. 
\end{proof}

\end{document}